\documentclass[letterpaper]{article}

\usepackage{style}

\usepackage[hyphens]{url}
\usepackage{graphicx}
\usepackage[numbers,sort&compress]{natbib}
\usepackage{caption}
\usepackage{booktabs}
\usepackage[table]{xcolor}
\usepackage{amsmath}
\usepackage{amssymb}
\usepackage{amsfonts}
\usepackage{amsthm}
\newtheorem{lemma}{Lemma}
\newtheorem{proposition}{Proposition}
\newtheorem{corollary}{Corollary}

\usepackage[colorlinks=true,linkcolor=blue,citecolor=blue,urlcolor=blue,filecolor=blue]{hyperref}

\title{CANDOR: Chance-Calibrated Neighborhood Discordance in Frozen Encoders for Medical Imaging}
\author{
    Soroosh Tayebi Arasteh\corresponding\textsuperscript{\rm 1,2},
    Sven Nebelung\textsuperscript{\rm 1,2},
    Daniel Truhn\textsuperscript{\rm 1,2}
}
\affiliations{
    \textsuperscript{\rm 1}Lab for AI in Medicine, RWTH Aachen University, Aachen, Germany\\
    \textsuperscript{\rm 2}Department of Diagnostic and Interventional Radiology, University Hospital RWTH Aachen, Aachen, Germany\\
    soroosh.arasteh@rwth-aachen.de
}

\begin{document}

\pagestyle{plain}

\maketitle
\thispagestyle{plain}

\begin{abstract}
In medical imaging, one frozen foundation encoder is often reused for many tasks. The encoder is usually chosen by the area under the receiver operating characteristic curve (AUROC) of a head trained on its features. The AUROC does not measure whether the encoder separates positive from negative images in its feature space. Chance-calibrated neighborhood discordance (CANDOR) measures this separation without training a head. Among images acquired the same way, a positive image is discordant when its $k$ nearest negatives are more similar to it on average than its $k$ nearest positives are. The discordance rate $D$ is the share of discordant positives. When the neighbors come from positive and negative candidate sets of equal size, $D$ has the chance level $1/2$ for features without information about the label. With unequal candidate sets, all 11 chest encoders have rates above $1/2$ (0.732 to 0.927) on the 12 chest findings of MIMIC-CXR and CheXpert. These rates increase as the finding becomes rarer. With equal sets, the same encoders have rates below $1/2$ (0.277 to 0.463). On the chest test split, linear heads miss 0.301 of the discordant positives and 0.165 of the other positives. Across 22 encoders, 605{,}443 images, and 8 tasks in 7 domains, the best encoder of every task has a rate below $1/2$. Because computing $D$ takes one forward pass per image, a hospital can compare candidate encoders before it trains any head.
\end{abstract}

\section*{Plain-Language Summary}

Many medical imaging systems are built on a foundation encoder, i.e., a large network that is trained once and then reused with fixed weights. Each new finding then needs only a small classifier on top of the encoder. So every such classifier depends on which images the fixed features can tell apart.

Practitioners usually train this classifier and compare its AUROC, a standard performance score. A high AUROC does not mean that the encoder had already placed images with the finding apart from images without it. Training a larger classifier changes the classifier but not those features.

For an image that has the finding, we ask whether the images closest to it have the finding or not. We compare only images taken the same way. We draw the same number of images with and without the finding. Then an encoder that has learned nothing is right half of the time. If every image is used instead, the larger group of images without the finding is closer by chance alone. Rare findings then look worse than they are.

We measured this on 22 encoders and 605{,}443 images, including chest radiographs, histopathology slides, fundus photographs, skin photographs, and mammograms. On every task, the best encoder is better than chance. Classifiers more often miss an image when the encoder places it among images without the finding. A result near chance does not mean that the features contain no information about the finding. For glaucoma, classifiers still detect the finding although the result is near chance.

Because nothing has to be trained first, a hospital can test a candidate encoder on its findings before committing to it. The measure describes an encoder on a task and does not predict which images a classifier will miss. Most of our evidence comes from chest radiography.

\section{Introduction}

\begin{figure*}[t]
\centering
\includegraphics[width=\textwidth]{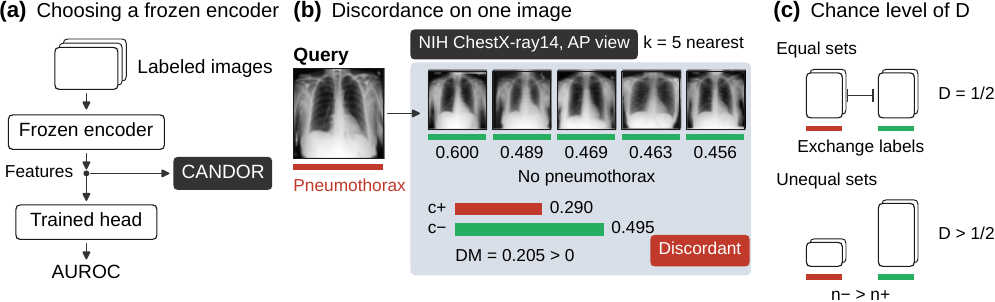}
\caption{CANDOR measured on one query. \textbf{(a)} the AUROC of a trained head and CANDOR on the frozen features; \textbf{(b)} a pneumothorax-positive radiograph of NIH ChestX-ray14 \citep{Wang2017ChestXray8}, its 5 nearest RAD-DINO neighbors with their cosine similarities, and the discordance margin (DM). This dataset is used only for this figure.}
\label{fig:overview}
\end{figure*}

A foundation encoder is pretrained once on a large image corpus and then reused with its weights frozen. For each new task, only a small head is trained on the features of the encoder. This setup is common in medical imaging, where labeled cases are scarce and one frozen encoder can be reused for many findings. General-purpose self-supervised encoders that saw no radiograph during training are also used for diagnostic tasks \citep{Arasteh2024NonMedicalSsl,Arasteh2025FederatedSsl}.

To choose an encoder, a practitioner trains a head on the frozen features for a target finding, scores the head by the area under the receiver operating characteristic curve (AUROC), and keeps the encoder whose head has the highest AUROC (Figure~\ref{fig:overview}a). AUROC measures how well the trained head predicts the finding. It does not measure whether the encoder places positive images apart from negative images in its feature space. As a Lipschitz function of the features, a head can separate a positive image from a negative image with nearly the same features only by a small normalized margin. Training changes the head but not the frozen features.

Like prior work on representation failure \citep{Tong2024EyesWideShut}, we examine nearest-neighbor relations in the frozen feature space. For each positive image, we ask whether its nearest negative images are more similar to it than its nearest positive images are. We report the share of such positives. To interpret this share, we need its value for features without information about the label. For most findings, a dataset has far more negative than positive images. A larger set drawn from the same distribution contains more similar nearest neighbors. So if the neighbors are searched among all images, the share is above $1/2$ for features without information about the label. It increases as the finding becomes rarer.

Instead, we draw a positive and a negative candidate set of equal size from images acquired the same way. The discordance rate $D$ is this share, computed with these sets. We call the method chance-calibrated neighborhood discordance, or CANDOR. For features without information about the label, exchanging the labels of the positive and negative sets leaves their joint distribution unchanged. So $D$ has the chance level $1/2$ for every encoder and finding. Computing $D$ requires one forward pass per image and no trained head.

Figure~\ref{fig:overview}b shows a discordant positive under RAD-DINO, an encoder trained on chest radiographs. It is an anteroposterior (AP) radiograph with pneumothorax from NIH ChestX-ray14 \citep{Wang2017ChestXray8}. Its 5 nearest neighbors among the AP radiographs of this dataset are pneumothorax-negative radiographs of other patients.

We apply CANDOR to 22 frozen encoders on 605{,}443 images from 20 public datasets, with 8 binary tasks in 7 domains: chest radiography, histopathology, fundus photography, dermatology, mammography, bird species, and aircraft variants. We make 3 contributions. \textbf{(i)} We define CANDOR and prove its chance level from the symmetry of exchanging the labels. On chest queries, the rate with unequal candidate sets is above $1/2$ for every encoder and increases as the finding becomes rarer. With equal sets, the rate of every encoder is below $1/2$. We detect no trend of this rate with the number of positives. \textbf{(ii)} We relate discordance to the errors of trained heads. A discordant positive is more likely to be missed by a head than another positive. Of the positives that are discordant under a majority of the encoders, most are classified correctly by at least one head. Yet the tested combinations of heads miss about as many of them as one head does, unless they predict positive for most negatives. CANDOR does not detect which images a head will miss. \textbf{(iii)} We report $D$ for every encoder on every task. It is lowest on tumor tiles. On glaucoma, the median encoder has a rate close to $1/2$, while linear heads still predict the finding above chance.

\section{Related Work}

\paragraph{Neighborhood Geometry and Representation Failure.} \citet{Tong2024EyesWideShut} identify CLIP-blind pairs, i.e., pairs of images mapped together by one encoder and separated by another. Because each pair is defined by 2 encoders, the pairs cannot be used to test whether either encoder alone separates the labels. The rate of such pairs also has no reference level that is the same across tasks. \citet{Muthyala2026FrozenSmallLesion} train logistic regression classifiers on frozen features of chest radiographs. With the class token as input, these classifiers are at chance for small perturbations inserted into radiographs. When only the patch tokens inside the annotated region are pooled, the classifiers detect these perturbations and also small lesions. Unlike their measure, $D$ requires neither a trained classifier nor an annotated region. \citet{CXRNeedsImage2026} occlude the annotated region of a finding to test whether a vision-language model still predicts the finding. Neural collapse \citep{Papyan2020NeuralCollapse} is a property of the features after classifier training. In contrast, CANDOR measures a frozen encoder before any head is trained.

\paragraph{Assessing Encoders Before Training a Head.} LogME \citep{You2021LogME} estimates the evidence of a linear model on the frozen features and the labels. RetMMD \citep{Dai2024RetMMD} estimates transferability to retrieval tasks with maximum mean discrepancy. Predictive $\mathcal{V}$-information \citep{Xu2020VInformation,Ethayarajh2022VUsable} and minimum description length probing \citep{Voita2020MdlProbing} measure how much label information a chosen family of classifiers can extract. R-precision and MAP@R \citep{Musgrave2020RealityCheck} score the retrieval of neighbors of the same class among all reference images. Their chance level depends on the class proportions. CANDOR compares the nearest positives with the nearest negatives within groups of images acquired the same way. With candidate sets of equal size, its chance level is $1/2$ on every task.

\paragraph{Failure Detection and Selective Prediction.} Unreliable predictions are commonly detected with maximum softmax probability \citep{Hendrycks2017Msp}, energy \citep{Liu2020Energy}, Mahalanobis distance \citep{Lee2018Mahalanobis}, deep ensemble disagreement \citep{Lakshminarayanan2017DeepEnsembles}, nearest-neighbor distance \citep{Sun2022KnnOod}, and selective classification \citep{Geifman2017Selective}. These methods mainly target out-of-distribution inputs or low-confidence predictions. A discordant image can be in distribution and still receive a confident prediction. We compare these methods with 2 label-free scores that are computed from the neighbors of an image.

\paragraph{Shared Representations and Shared Failure.} Representations of different architectures and objectives appear to be converging \citep{Huh2024Platonic}. Their similarity is routinely quantified \citep{Kornblith2019Cka}. If different encoders fail on the same images, an ensemble or a selector across encoders has fewer errors to correct. Coordinated failure across a set of models has also been studied outside vision \citep{Arasteh2026CollectiveReliability}. Shortcut learning can make predictions depend on acquisition-related confounders \citep{Geirhos2020Shortcut,DeGrave2021Shortcut}. CANDOR compares only images that share the recorded acquisition variables.

\section{The CANDOR Framework}

\subsection{Problem Formulation}
An encoder $g:\mathcal{X}\to\mathbb{S}^{d-1}$ maps an image to a feature vector of length 1 and is kept frozen. The inner product between 2 vectors is then the cosine similarity between the corresponding images. A head $h:\mathbb{S}^{d-1}\to\mathbb{R}$ maps the features to a binary prediction for $F:\mathcal{X}\to\{0,1\}$ by thresholding $h(z)$ at $\tau$, i.e., $\hat y(z)=\mathbf{1}[h(z)\ge\tau]$. Only the head is trained.

A \emph{matched group} $s$ contains images acquired under the same conditions and is not defined by the label. For chest radiographs, it is one site and one view. For the other medical datasets, the matched group is the source dataset. For birds and aircraft, it is the family or the manufacturer. Within a matched group, these conditions cannot separate the labels.

Inside a matched group, we draw a positive candidate set $R^{+}_{s}$ from the positives and a negative candidate set $R^{-}_{s}$ of the same size $n$ from the negatives. For a positive query image $x$, let $N_k(x;R)$ be its $k$ nearest neighbors in $R$, excluding $x$ where it is a member, and write
\begin{equation}
\begin{aligned}
c^{\pm}(x) \;&=\; \frac{1}{k}\sum_{z\in N_k(x;\,R^{\pm}_{s})}\big\langle g(x),\,g(z)\big\rangle ,\\[2pt]
\mathrm{DM}_g(x) \;&=\; c^{-}(x)-c^{+}(x) .
\end{aligned}
\label{eq:dm}
\end{equation}
Here, $c^{+}$ is the mean similarity to the $k$ nearest other positives, and $c^{-}$ is the mean similarity to the $k$ nearest negatives. We call their difference $\mathrm{DM}_g$ the \emph{discordance margin}. We call $x$ \emph{discordant} when $\mathrm{DM}_g(x)>0$. Equivalently, the $k$-nearest-neighbor rule that assigns $x$ to the candidate set with the higher mean similarity misclassifies $x$. The 5 nearest neighbors of the radiograph in Figure~\ref{fig:overview}b all come from $R^{-}_{s}$. Its $c^{-}=0.495$ is larger than its $c^{+}=0.290$. The \emph{discordance rate} is the share of positives that are discordant,
\begin{equation}
D(g;F,s)=\Pr\big[\,\mathrm{DM}_g(x)>0 \;\big|\; F(x)=1,\; x\in s\,\big].
\label{eq:operator}
\end{equation}

\subsection{The Discordance Rate and Its Chance Level}
Write $H_0$ for the case where, given $s$, the features are independent of $F$ and exchangeable across images. Under $H_0$, exchanging the labels of the 2 candidate sets leaves their joint distribution unchanged (Figure~\ref{fig:overview}c).

\begin{lemma}[Chance level]
\label{lem:chance}
Fix a matched group $s$ and a positive query $x$ in neither candidate set. Draw $R^{+}_{s}$ and $R^{-}_{s}$ of equal size $n$ uniformly from the positives and the negatives of $s$, and let $t=\Pr[\mathrm{DM}_g(x)=0]$. Under $H_0$, $\Pr[\mathrm{DM}_g(x)>0]=(1-t)/2$, which is $1/2$ when similarities do not tie.
\end{lemma}

\begin{proof}
Under $H_0$, the $2n$ images of the 2 sets are exchangeable draws from the same feature distribution. Then exchanging the sets leaves the joint law of $g(x)$, $g(R^{+}_{s})$, and $g(R^{-}_{s})$ unchanged. The exchange sends $\mathrm{DM}_g(x)$ to $-\mathrm{DM}_g(x)$, so $\Pr[\mathrm{DM}_g>0]=\Pr[\mathrm{DM}_g<0]=(1-t)/2$.
\end{proof}

Since $D$ includes only $\mathrm{DM}_g>0$, a tie is treated as concordant. A query that is a member of $R^{+}_{s}$ is removed from it after the draw. It is then compared with $n-1$ positives and $n$ negatives. Removing one random negative would restore equal sizes. This removal changes $c^{-}$ only if the removed image is one of the $k$ nearest negatives. Under $H_0$, this event has probability $k/n$. Hence the chance level of such a query is at most $1/2+k/n$, i.e., 0.501 for $k=5$ and $n=4{,}000$.

\paragraph{Unequal Candidate Sets.} Suppose instead that the candidates are all positives and all negatives of the matched group, with $n_{-}>n_{+}$ for a rare finding. We write the rate with these sets $D_{\mathrm{all}}$. Build the larger set by adding $n_{-}-n_{+}$ negatives to a draw of $n_{+}$. Adding candidates can only increase each of the $k$ largest similarities, so $c^{-}(x)$ is at least its value on the first $n_{+}$ negatives. Under $H_0$, the comparison of the first $n_{+}$ negatives with the positives is at chance. Hence $D_{\mathrm{all}}>1/2$ even for features that are independent of $F$. The excess increases with the odds $n_{-}/n_{+}$, i.e., as the prevalence $n_{+}/(n_{+}+n_{-})$ decreases.

\subsection{The Margin Bound}
A head is $L$-Lipschitz when its output changes by at most $L$ times the distance between 2 inputs. Every linear head and every network with Lipschitz activations is $L$-Lipschitz for some $L$.

\begin{proposition}[Margin constraint]
\label{prop:irr}
Let $h$ be $L$-Lipschitz on $\mathbb{S}^{d-1}$ with decision rule $\mathbf{1}[h(z)\ge\tau]$, and say $h$ separates an opposite-label pair $(z^{+},z^{-})$ with margin $m>0$ if $h(z^{+})\ge\tau+m$ and $h(z^{-})\le\tau-m$. Then
\[
\lVert z^{+}-z^{-}\rVert \;\ge\; \frac{2m}{L},
\qquad\text{and}\qquad
\langle z^{+},z^{-}\rangle \;\le\; 1-\frac{2m^{2}}{L^{2}} .
\]
\end{proposition}

\begin{proof}
The margin conditions give $h(z^{+})-h(z^{-})\ge 2m$, and the Lipschitz property gives $h(z^{+})-h(z^{-})\le L\lVert z^{+}-z^{-}\rVert$. On the unit sphere, $\lVert z^{+}-z^{-}\rVert^{2}=2-2\langle z^{+},z^{-}\rangle$.
\end{proof}

The same inequality underlies certified robustness \citep{Tsuzuku2018LipschitzMargin}.

\begin{corollary}[Margin bound for a positive image]
\label{cor:cap}
Let $x$ be a positive. Some negative $z^{-}$ in its matched group has $\langle g(x),g(z^{-})\rangle\ge c^{-}(x)$, because a maximum is at least a mean. Every $L$-Lipschitz head that separates $g(x)$ from $g(z^{-})$ with margin $m$ obeys
\begin{equation}
\frac{m}{L}\ \le\ \tfrac{1}{2}\sqrt{\,2-2\,c^{-}(x)\,}.
\label{eq:cap}
\end{equation}
\end{corollary}

The ratio $m/L$ does not change when the head is rescaled. For example, $c^{-}(x)=0.95$ implies $m/L\le 0.16$. The bound depends only on $c^{-}(x)$. For a discordant positive, $c^{-}(x)$ exceeds $c^{+}(x)$. The negatives that limit the margin are then more similar to $x$ than its nearest positives are. The bound constrains the normalized margin and not correctness, since 2 distinct feature vectors can always be separated. It applies to one Lipschitz head on the features of one encoder. It does not apply to a fusion or an ensemble of encoders, to a joint head over several encoders, to other features of the same encoder, or to a rule that is not Lipschitz.

\subsection{Derived Measures}
\paragraph{The Discordant Set.} The set of discordant positives is $\mathcal{D}_F(g)=\{x: F(x)=1,\ \mathrm{DM}_g(x)>0\}$. If encoders share discordant images more often than by chance, discordance is not specific to one encoder. We test the overlap of 2 such sets with the hypergeometric tail, i.e., the probability that 2 random sets of the same sizes share at least this overlap.

\paragraph{Occlusion Retention.} Where an expert annotated the evidence region $b_F$ of a positive $x$, write $x\ominus b_F$ for the image with that region replaced by the mean pixel value, and set
\begin{equation}
\rho_g(x)=\big\langle\, g(x),\; g(x\ominus b_F)\,\big\rangle .
\label{eq:retention}
\end{equation}
Retention near 1 means that occluding the annotated region changes the features little.

\paragraph{Neighborhood Impurity.} To find discordant images without the label of the query, we use the labels of its training neighbors. Fix a candidate set from the training split of the matched group of $x$, let $\hat p_k(x,F)$ be the fraction of the $k$ nearest training neighbors of $x$ that are positive for $F$, and set
\begin{equation}
r(x,F)=1-\big\lvert\,2\hat p_k(x,F)-1\,\big\rvert .
\label{eq:impurity}
\end{equation}
The score $r$ is high when the $k$ neighbors include positives and negatives. It is 0 when all $k$ neighbors share a label, including a positive whose neighbors are all negatives.

\section{Experiments}

\paragraph{Scope.} Every task is the binary detection of one finding. Table~\ref{tab:scope} lists the datasets, the positives of each task, and the encoders. A positive is an image labeled with the finding, and the negatives are the other labeled images of the same matched group. Birds and aircraft are tested one class at a time, with the other classes of the same family or manufacturer as negatives. Such a test is possible for the 167 of 200 species and the 87 of 100 variants that share their family or manufacturer with another class. For these 2 tasks, $D$ is the unweighted mean over those classes. The chest radiographs come from 5 sites and form 11 matched groups. Their labels are mapped to a common set of 12 findings. We also use the external site Open-I and the expert bounding boxes of MS-CXR and RSNA. For the main chest analysis, we use MIMIC-CXR and CheXpert, which are the only sites with all 12 findings. This analysis has 22{,}133 positive queries. The encoders range from 7M to 1.1B parameters, were released from 2020 to 2025, and use 3 objectives: self-supervised (SSL), image-text (I-T), and supervised (Sup.). Every image is stored at $224 \times 224$ pixels. The image processor of each encoder then resizes it to the input size of that encoder, e.g., $518 \times 518$ pixels for RAD-DINO and $512 \times 512$ pixels for SigLIP-2.

\begin{table*}[t]
\centering
\caption{The datasets, the positives of each task, and the encoders.}
\label{tab:scope}
\setlength{\tabcolsep}{5pt}
\scriptsize
\begin{tabular}{@{}lrp{3.75in}p{1.75in}@{}}
\toprule
Domain & Images & Datasets & Positives \\
\midrule
Chest radiography & 551{,}969 & MIMIC-CXR \citep{Johnson2019Mimic}, CheXpert \citep{Irvin2019Chexpert}, PadChest \citep{Bustos2020Padchest}, VinDr-CXR \citep{Nguyen2022VindrCxr}, VinDr-PCXR \citep{Pham2023VindrPcxr}, Open-I \citep{Demner2016OpenI}, RSNA \citep{Shih2019Rsna}, MS-CXR \citep{Boecking2022MsCxr} & Each of 12 findings \\
Histopathology & 20{,}000 & NCT-CRC \citep{Kather2018NctCrc}, PatchCamelyon \citep{Veeling2018Pcam} & Tumor tiles \\
Bird species & 11{,}788 & CUB-200-2011 \citep{Wah2011Cub} & One species within its family \\
Aircraft variants & 10{,}000 & FGVC-Aircraft \citep{Maji2013Fgvc} & One variant within its manufacturer \\
Fundus & 6{,}030 & Harvard-FairVision \citep{Luo2024FairVision}, APTOS \citep{Aptos2019}, Messidor-2 \citep{Decenciere2014Messidor}, IDRiD \citep{Porwal2018Idrid} & Glaucoma; referable diabetic retinopathy, grade 2 or higher \\
Dermatology & 4{,}656 & Fitzpatrick17k \citep{Groh2021Fitzpatrick}, ISIC 2019 \citep{Combalia2019Isic}, DDI \citep{Daneshjou2022DDI} & Malignant lesions, with benign lesions as negatives \\
Mammography & 1{,}000 & VinDr-Mammo \citep{Nguyen2023VindrMammo} & Mass or suspicious calcification \\
\midrule
Encoder family & Count & \multicolumn{2}{@{}l@{}}{Models} \\
\midrule
Chest & 3 & \multicolumn{2}{@{}p{5.58in}@{}}{RAD-DINO \citep{PerezGarcia2025RadDino}, BiomedCLIP \citep{Zhang2024BiomedClip}, DenseNet-121 \citep{Cohen2022Torchxrayvision}} \\
Histopathology & 7 & \multicolumn{2}{@{}p{5.58in}@{}}{Phikon-v2 \citep{Filiot2024PhikonV2}, UNI, UNI2 \citep{Chen2024Uni}, Virchow \citep{Vorontsov2024Virchow}, Virchow2 \citep{Zimmermann2024Virchow2}, Prov-GigaPath \citep{Xu2024ProvGigapath}, CONCH \citep{Lu2024Conch}} \\
Fundus & 2 & \multicolumn{2}{@{}p{5.58in}@{}}{RETFound \citep{Zhou2023Retfound}, FLAIR \citep{SilvaRodriguez2025Flair}} \\
Dermatology & 2 & \multicolumn{2}{@{}p{5.58in}@{}}{MONET \citep{Kim2024Monet}, DermLIP \citep{Yan2025Derm1M,Yan2025Panderm}} \\
General purpose & 7 & \multicolumn{2}{@{}p{5.58in}@{}}{DINOv3-S/B/L/H+ \citep{Simeoni2025Dinov3}, DINOv2-L \citep{Oquab2024Dinov2}, CLIP-L/14 \citep{Radford2021Clip}, SigLIP-2-L \citep{Tschannen2025Siglip2}} \\
Untrained & 1 & \multicolumn{2}{@{}p{5.58in}@{}}{Randomly initialized ViT-L} \\
\bottomrule
\end{tabular}
\end{table*}

\paragraph{Implementation.} We use the last-layer class token of every vision transformer trained without text and of the randomly initialized ViT-L. For the image-text models, we use the image embedding after the projection, where the model has a projection. For DenseNet-121, we use the pooled penultimate layer. We use $k=5$ neighbors. The positive and negative candidate sets of a matched group have the same size of at most 4{,}000 images. The queries are the positives in a random sample of at most 4{,}000 images of each matched group. The sets are drawn once, with seed 42, and are kept fixed in every bootstrap resample. Because they are drawn from all images of the matched group, another image of the same patient can be a neighbor of the query. Each linear head is trained with Adam \citep{Kingma2015Adam} for 30 epochs with learning rate 0.001, weight decay $10^{-4}$, and batch size 256. Each head is trained on at most 25{,}000 images and tested on at most 20{,}000 images of its domain. These images are drawn once from a stable hash and are the same for every encoder, so that every comparison across encoders is paired. We fix the operating point of each finding at 0.800 sensitivity on the validation split. The deep ensemble has 5 members.

\paragraph{Reporting.} The chance level of $D$ is 0.500 by Lemma~\ref{lem:chance}. Retention is a cosine similarity with no fixed reference level. A chest rate is pooled over sites by sample size and then over the 12 findings by a mean weighted for clinical severity. The weight is 3 for pneumothorax, pneumonia, consolidation, edema, and lung lesion; 2 for pleural effusion, fracture, cardiomegaly, lung opacity, and enlarged cardiomediastinum; and 1 for atelectasis and pleural other. A chest AUROC is the mean over the 12 findings. A $\pm$ value is the standard deviation over 1{,}000 bootstrap resamples clustered by patient. An interval is the 95\% interval of the same resamples. We correct each family of $p$-values with Benjamini-Hochberg false discovery rate (FDR) control at 0.05 \citep{Benjamini1995Fdr}.

\paragraph{Baselines.} Beside $D$, we report 3 scores for each encoder. The margin classifier scores every positive and every negative of a matched group by $c^{+}-c^{-}$, where each image is excluded from the candidate set containing it. Its AUROC has a chance level of 0.500. The test AUROC of the trained linear head is the usual criterion for choosing an encoder. LogME \citep{You2021LogME} scores the frozen features without a trained head. As baselines for $r$, we use the standard failure detection scores: confidence (maximum softmax probability, maximum logit, and energy), deep ensemble disagreement, Mahalanobis distance, and $k$-nearest-neighbor distance. Each is computed on the same images as $r$. Every dataset and every pretrained encoder is public.

\subsection{Chest Encoders and the Choice of Candidate Sets}

On MIMIC-CXR and CheXpert (Table~\ref{tab:panel}), RAD-DINO has the lowest discordance and retention and the highest margin AUROC. DenseNet-121 has the highest head AUROC. The encoder rankings by head AUROC and by $D$ are similar (Spearman $-0.855$), although a different encoder is first in each ranking. A head can have a high AUROC while a minority of the positives are discordant. For example, the linear head on RAD-DINO has an AUROC of 0.826 on pneumothorax, while 0.184 of the pneumothorax positives are discordant under RAD-DINO.

\begin{table}[tb]
\centering
\caption{Chest encoders on MIMIC-CXR and CheXpert. Horizontal lines separate the 3 chest specialists, the 7 general-purpose encoders, and the shaded untrained encoder. $D_{\mathrm{all}}$ is computed with unequal candidate sets. Bold marks rank 1, and underlining marks rank 2.}
\label{tab:panel}
\setlength{\tabcolsep}{2.5pt}
\scriptsize
\begin{tabular}{@{}llrccccc@{}}
\toprule
& & Params & \multicolumn{2}{c}{AUROC$\uparrow$} & \multicolumn{2}{c}{Discordance} & \\
\cmidrule(lr){4-5}\cmidrule(lr){6-7}
Encoder & Objective & (M) & Head & Margin & $D\downarrow$ & $D_{\mathrm{all}}$ & Retention$\downarrow$ \\
\midrule
RAD-DINO       & SSL & 86 & \underline{0.718} & \textbf{0.753} & \textbf{0.277} & 0.732 & \textbf{0.622} \\
DenseNet-121   & Sup. & 7 & \textbf{0.722} & \underline{0.706} & \underline{0.338} & 0.823 & \underline{0.845} \\
BiomedCLIP     & I-T & 86 & 0.695 & 0.681 & 0.375 & 0.849 & 0.880 \\
\midrule
DINOv3-S       & SSL & 21 & 0.664 & 0.649 & 0.405 & 0.882 & 0.948 \\
SigLIP-2-L     & I-T & 307 & 0.655 & 0.643 & 0.406 & 0.881 & 0.953 \\
DINOv3-B       & SSL & 86 & 0.660 & 0.650 & 0.421 & 0.882 & 0.972 \\
DINOv3-H+      & SSL & 840 & 0.644 & 0.623 & 0.423 & 0.901 & 0.991 \\
CLIP-L/14      & I-T & 307 & 0.635 & 0.614 & 0.424 & 0.901 & 0.962 \\
DINOv2-L       & SSL & 307 & 0.645 & 0.628 & 0.428 & 0.901 & 0.984 \\
DINOv3-L       & SSL & 300 & 0.656 & 0.642 & 0.428 & 0.893 & 0.983 \\
\midrule
\rowcolor{gray!12}
Random ViT-L   & None & 307 & 0.597 & 0.563 & 0.463 & 0.927 & 0.976 \\
\bottomrule
\end{tabular}
\end{table}

To calibrate $D$ at the ends of its range, we build 200 images that differ only by a small local marker that defines the class. We then construct a blurring encoder that is invariant to the marker and an encoder that is sensitive to it. The invariant encoder has $D=0.495\pm0.035$. Its 95\% interval contains 0.500, i.e., the chance level of Lemma~\ref{lem:chance}. The marker-sensitive encoder has $D=0.000$, i.e., the lowest possible value. The chest encoders have rates from 0.277 to 0.463. Each of them is closer to the invariant reference than to the marker-sensitive reference. For all 132 combinations of the 11 encoders and the 12 findings, the AUROC of the margin classifier has a 95\% bootstrap interval above 0.500. So every chest encoder separates positive from negative images above chance for every finding.

\begin{figure}[t]
\centering
\includegraphics[width=\columnwidth]{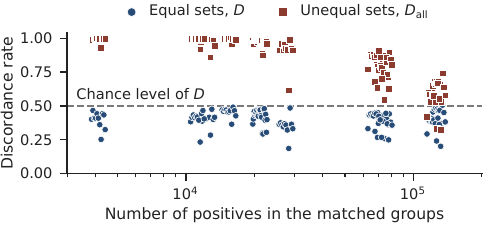}
\caption{Discordance of the 11 chest encoders on the 12 findings of MIMIC-CXR and CheXpert, with equal ($D$) and unequal ($D_{\mathrm{all}}$) candidate sets. Each marker is one encoder on one finding.}
\label{fig:auroc}
\end{figure}

To test the role of the equal set size, we recompute the rate on the same queries with every positive and every negative of the matched group as candidates, i.e., $D_{\mathrm{all}}$. With these unequal sets, every chest encoder has a rate above 0.500, from 0.732 for RAD-DINO to 0.927 for the randomly initialized ViT-L (Table~\ref{tab:panel}). Over the 132 combinations, the median is 0.929 for $D_{\mathrm{all}}$ and 0.410 for $D$. The rate $D_{\mathrm{all}}$ decreases as the number of positives increases (Spearman $-0.883$). For $D$, this correlation is $-0.084$ (Figure~\ref{fig:auroc}). With the labels shuffled within each matched group, the median is 0.996 with unequal sets and 0.498 with equal sets. With unequal sets, the rate depends on the prevalence of the finding. The encoder rankings under $D_{\mathrm{all}}$ and under $D$ still have a Spearman correlation of 0.891.

Averaged over the 12 findings, the encoder ranking by head AUROC has a Spearman correlation of 0.513 with the ranking by $-D$ and of 0.462 with the ranking by LogME.

We detect no association of $D$ with release year (slope $p=0.436$ over the 10 trained encoders) or with finding size (slope $p=0.479$ over the 132 combinations). Within the DINOv3 family, $D$ is from 0.405 to 0.428 over a 40-fold range of parameter counts. To test the effect of the pretraining objective, we trained DINOv3-S and DINOv3-B from their initializations with 3 objectives each, on the same MIMIC-CXR training images. The mean discordance of the resulting 6 encoders is 0.446 with image-text, 0.378 with self-supervised contrastive, and 0.329 with supervised training. With 2 backbones and 3 findings per objective, this comparison is underpowered.

\subsection{Discordance Across Domains}

\begin{table}[tb]
\centering
\caption{$D$ of every encoder on the tasks outside chest radiography, with task-level results in the last 5 rows. Bold marks the lowest $D$. Sig.\ is the share of encoder pairs with a significant overlap of discordant sets after FDR correction.}
\label{tab:domains}
\setlength{\tabcolsep}{3pt}
\scriptsize
\begin{tabular}{@{}lccccccc@{}}
\toprule
& \rotatebox{90}{Tumor} & \rotatebox{90}{Aircraft} & \rotatebox{90}{Birds} & \rotatebox{90}{Retinopathy} & \rotatebox{90}{Skin} & \rotatebox{90}{Glaucoma} & \rotatebox{90}{Mammogram} \\
\midrule
Phikon-v2      & 0.039 &  &  &  &  &  &  \\
UNI            & 0.037 &  &  &  &  &  &  \\
UNI2           & \textbf{0.027} &  &  &  &  &  &  \\
Virchow        & 0.036 &  &  &  &  &  &  \\
Virchow2       & 0.035 &  &  &  &  &  &  \\
Prov-GigaPath  & \textbf{0.027} &  &  &  &  &  &  \\
CONCH          & 0.063 &  &  &  &  &  &  \\
\midrule
RETFound       &  &  &  & 0.187 &  & 0.481 &  \\
FLAIR          &  &  &  & \textbf{0.099} &  & \textbf{0.431} &  \\
\midrule
MONET          &  &  &  &  & 0.202 &  &  \\
DermLIP        &  &  &  &  & \textbf{0.187} &  &  \\
\midrule
DINOv3-S       & 0.090 & 0.073 & 0.059 & 0.198 & 0.195 & 0.544 & 0.496 \\
DINOv3-B       & 0.075 & 0.052 & 0.051 & 0.187 & 0.213 & 0.548 & 0.498 \\
DINOv3-L       & 0.073 & 0.058 & \textbf{0.045} & 0.177 & 0.206 & 0.498 & 0.478 \\
DINOv3-H+      & 0.066 & 0.080 & 0.058 & 0.161 & 0.225 & 0.522 & 0.580 \\
DINOv2-L       & 0.075 & 0.127 & 0.054 & 0.166 & 0.251 & 0.523 & 0.524 \\
CLIP-L/14      & 0.077 & 0.120 & 0.063 & 0.168 & 0.233 & 0.493 & \textbf{0.464} \\
SigLIP-2-L     & 0.066 & \textbf{0.035} & 0.056 & 0.176 & 0.217 & 0.502 & 0.500 \\
\midrule
Random ViT-L   & 0.117 & 0.441 & 0.436 & 0.244 & 0.373 & 0.461 & 0.536 \\
\midrule
Positive queries & 1{,}418 & 8{,}700 & 9{,}837 & 1{,}300 & 1{,}208 & 1{,}016 & 500 \\
$\Delta J$ & 0.270 & 0.104 & 0.312 & 0.248 & 0.240 & 0.131 & 0.117 \\
Sig.\ (\%) & 100.0 & 13.5 & 24.3 & 100.0 & 100.0 & 100.0 & 100.0 \\
$r$ & 0.659 & 0.601 & 0.419 & 0.639 & 0.618 & 0.550 & 0.523 \\
Conf. & 0.919 & 0.846 & 0.826 & 0.657 & 0.698 & 0.629 & 0.556 \\
\bottomrule
\end{tabular}
\end{table}

Table~\ref{tab:domains} lists $D$ for every encoder on the 7 tasks outside chest radiography. Even the randomly initialized ViT-L separates the classes of some tasks. It has $D=0.117$ on histopathology tiles and 0.461 on glaucoma. We compare each rate with the rate of this untrained encoder on the same task as well as with 0.500. On glaucoma, the median $D$ over 10 encoders is 0.500. On mammography, it is 0.499 over 8 encoders. On glaucoma, the test AUROCs of the linear heads are 0.639 to 0.717 for the trained encoders and 0.585 for the untrained encoder. So a rate near 0.500 does not mean that the features contain no information about the finding.

The discordant sets of encoder pairs overlap above chance for 655 of the 660 chest pairs and for all 268 pairs in dermatology, fundus, histopathology, and mammography. After FDR correction, the overlap remains significant for 606 chest pairs and for all 268 others. On natural images, 51.3\% of the bird pairs and 50.2\% of the aircraft pairs overlap above chance. Table~\ref{tab:domains} also lists $\Delta J$, i.e., the mean Jaccard overlap in excess of chance, and the AUROCs of $r$ and of confidence (Conf.) for detecting head errors. Shared label noise, shared case difficulty, related architectures, and the common preprocessing can each produce such overlap. The pairwise tests are also not independent.

\subsection{Discordance and Head Errors}

Trained heads miss discordant positives more often than other positives. On the test split of the 5 chest sites, the heads of the 11 encoders miss 0.301 of the discordant positives and 0.165 of the others. Stratified by encoder and finding, the Mantel-Haenszel difference of the miss rates is $0.136\pm0.002$. It is positive for every encoder, from 0.038 for the randomly initialized ViT-L to 0.300 for DenseNet-121. Among the test negatives, those that are nearer to the positives on average have a false positive rate of 0.676. The other negatives have 0.441. Because computing $\mathrm{DM}_g$ requires the label of the image, CANDOR is a geometric diagnostic of an encoder on a task.

\begin{table}[tb]
\centering
\caption{Six rules that combine the 11 chest heads, scored by the miss rate on positives and the false positive (FP) rate on negatives. Columns 263 and 33{,}770 are the test positives that are discordant under most encoders.}
\label{tab:selection}
\setlength{\tabcolsep}{3pt}
\scriptsize
\begin{tabular}{@{}lcccc@{}}
\toprule
& \multicolumn{3}{c}{Miss rate$\downarrow$} & FP rate$\downarrow$ \\
\cmidrule(lr){2-4}\cmidrule(lr){5-5}
Rule & 263 & 33{,}770 & All & Negatives \\
\midrule
Single head (mean of 11) & 0.426 & 0.337 & 0.218 & 0.527 \\
Head chosen on validation & 0.445 & 0.380 & 0.206 & 0.434 \\
Majority vote & 0.418 & 0.307 & 0.186 & 0.544 \\
Mean probability of 11 heads & 0.494 & 0.388 & 0.217 & 0.444 \\
Stacked logistic regression & 0.433 & 0.396 & 0.205 & 0.418 \\
OR rule & 0.049 & 0.045 & 0.019 & 0.880 \\
\bottomrule
\end{tabular}
\end{table}

In Table~\ref{tab:selection}, we combine the 11 heads with 6 rules. Every rule is scored on the same 20{,}000 test images, since all heads are tested on this sample. Of the positive queries in this sample, 263 are discordant under a majority of the 11 encoders. In the whole test split of the 5 sites, 33{,}770 positives are discordant under a majority. A single head misses 0.426 of the 263 on average. The OR rule predicts positive whenever any head predicts positive. It misses $0.049\pm0.013$ of the 263 but predicts positive for 0.880 of the test negatives. The other 4 rules miss 0.418 to 0.494 of the 263. Among them, the head chosen by validation AUROC misses $0.445\pm0.032$. So the heads err on different images. Yet every combination other than the OR rule misses about as many of these positives as one head does. On all test positives, the head chosen on validation misses 0.206. The operating point of 0.800 sensitivity implies a miss rate of 0.200.

To test $r$ on chest, we train each head on MIMIC-CXR and take the candidate set of $r$ from the MIMIC-CXR training split in the view of the query. The neighborhood impurity $r$ of Eq.~\ref{eq:impurity} then detects the errors of the head with a mean AUROC of 0.522 over 573 combinations of encoder, finding, and test site. Confidence has a mean AUROC of 0.868 on the same images. After FDR correction, confidence is significantly better in 538 of the 573 combinations and significantly worse in 7. The deep ensemble, $k$-nearest-neighbor distance, and Mahalanobis distance have mean AUROCs of 0.759, 0.677, and 0.659. As a second proxy, we compute $\mathrm{DM}_g$ with the prediction of the head in place of the label. This proxy detects the errors of the head with a mean AUROC of 0.652 over the 10 trained encoders and the 12 findings of the 5 chest sites. Maximum softmax probability has 0.791 on the same images. In all 7 domains, $r$ has a lower AUROC than confidence (Table~\ref{tab:domains}). The negative result is limited to these 2 label-free proxies.

\subsection{Occlusion Retention and Discordance}

Chest retention is the mean over the boxed pneumothorax and pneumonia positives. Across the 10 trained chest encoders, it has a Spearman correlation of 0.891 with $D$ ($p=0.0005$). Because encoders differ in the cosine similarity of random image pairs, we also standardize retention as $(\rho_g-c)/(1-c)$, where $c$ is the mean cosine similarity of random pairs of the boxed images. The correlation is then 0.455 ($p=0.19$). We report the relation only as an association. As a control, we occlude a box of the same size at the mirrored position, or at a random position when the mirrored box overlaps the annotated box. Over the 10 encoders, retention decreases by 0.090 on average for the annotated box and by 0.062 for the control box. For 10 of the 11 encoders, occluding the annotated box lowers retention more than occluding the control box.

On fundus, we detect no association between $D$ and the retention on IDRiD over 10 encoders. The Spearman correlation is $-0.479$ ($p=0.16$) for glaucoma and 0.304 ($p=0.39$) for referable retinopathy. On IDRiD, occluding an expert-annotated lesion changes the features of RETFound about as little as those of the randomly initialized ViT-L. Across the 4 lesion types, the retention is 0.940 to 0.974 for RETFound and 0.954 to 0.982 for the untrained encoder.

\subsection{Ablation Studies}

\paragraph{(1) Box-Verified Positives} reduce label ambiguity. On the 233 pneumothorax positives with MS-CXR boxes, every encoder has $D$ at or below 0.500. RAD-DINO has $D=0.167$ on them and 0.184 on all its pneumothorax positives.

\paragraph{(2) Removing the Matching} changes $D$ in either direction. Over the 22 combinations of the 11 encoders with atelectasis and cardiomegaly, pooling the candidate sets across sites and views changes $D$ by 0.046 on average and by at most 0.123. The matched rate is lower than the pooled rate in 9 of the 22 combinations.

\paragraph{(3) The Confidence Baselines} reduce to one detector, since maximum softmax probability, maximum logit, and energy are monotone transforms of the single logit of a binary head. Their AUROCs agree to within $10^{-6}$ on the 573 chest combinations and differ by at most 0.005 on all 2{,}532 combinations.

\paragraph{(4) Varying the Candidate Sets} changes an encoder's $D$ by at most 0.053 (at $k=1$). With $k$ from 1 to 20 and sets of 250 to 4{,}000 images, the median $D$ over the chest combinations is from 0.399 to 0.435. The Spearman correlation of the encoder ranking with the reported ranking is at least 0.827. Across 10 redraws of the sets, an encoder's $D$ changes by at most 0.014. Removing one random negative whenever a query is removed from its positive set keeps the 2 sets of equal size and changes an encoder's $D$ by at most 0.0004. Excluding the patient of the query from the candidate sets changes $D$ by at most 0.041. Ties between $c^{+}$ and $c^{-}$ occur for 0.002\% of the chest queries.

\paragraph{(5) Input Resolution and Pooling} affect the transfer of frozen features \citep{Arasteh2025Dinov3Resolution,Muthyala2026FrozenSmallLesion}. We test them on a fixed sample of MIMIC-CXR and CheXpert, with up to 500 positives and 500 negatives per matched group and finding. With source images of $512 \times 512$ pixels, $D$ is 0.300 for RAD-DINO and 0.414 for DINOv3-L. At $224 \times 224$ pixels, it is 0.301 and 0.438. With the mean of the patch tokens in place of the class token, an encoder's $D$ changes by $-0.020$ to $+0.079$. The encoder ranking then has a Spearman correlation of 0.915 with the ranking by class token.

\paragraph{(6) Restricting to Isolated Findings} changes the median $D$ over the 132 combinations from 0.410 to 0.417. Here, the positives have none of the other 11 findings, and the negatives have none of the 12. The encoder ranking has a Spearman correlation of 0.945 with the reported ranking.

\paragraph{(7) Negative Queries} have a median discordance rate of 0.388 over the 132 combinations when the roles of the 2 sets are exchanged.

\section{Limitations and Conclusion}
Of the 605{,}443 images, 551{,}969 are chest radiographs, because chest radiography has the largest public labeled datasets. A rate outside chest describes only the 1 to 4 public datasets of its domain. More public datasets per modality are needed to describe a modality. We match only on the recorded covariates. An unrecorded difference in acquisition can still contribute to $D$ within a matched group. Datasets with richer acquisition records would allow finer groups. We measured the association of $D$ with occlusion retention on 10 chest encoders only. After standardization, the Spearman correlation decreases from 0.891 to 0.455. Retention also requires expert-annotated regions, which exist for 2 of the 8 tasks. More encoders and annotated datasets are needed to test this association. We evaluate only frozen encoders and binary findings. The effect of finetuning on $D$ and a version of $D$ for multi-class or ordinal labels remain to be studied.

With candidate sets of equal size, the chance level of the discordance rate is $1/2$. With unequal sets, the rate depends on how rare the finding is. On every task, the best of the 22 encoders has a rate below $1/2$. Trained heads miss more of the discordant positives than of the other positives. Yet $D$ cannot be used to predict which images a head will miss. With one forward pass per image, a practitioner can compute $D$ for every candidate encoder before training a head and again whenever a new encoder is released.

\section*{Acknowledgments}
STA is supported by the German Federal Ministry of Research, Technology and Space (ARISTOTLE, 01ZU2602B) and the Excellence Strategy of the German Federal Government, the L\"ander, and RWTH ERS (START\_526-26). SN was supported by grants from the Deutsche Forschungsgemeinschaft (DFG) (NE 2136/3-1, LI 3893/6-1, and TR 1700/7-1). DT was supported by grants from the DFG (NE 2136/3-1, LI 3893/6-1, and TR 1700/7-1) and is supported by the German Federal Ministry of Education (TRANSFORM LIVER, 031L0312A; SWAG, 01KD2215B) and the European Union's Horizon Europe research and innovation programme (ODELIA, Open Consortium for Decentralized Medical Artificial Intelligence, 101057091).

\bibliography{arxiv}

\end{document}